\documentclass[11pt]{article}
\usepackage{amsmath, amssymb, amsthm}
\usepackage{graphicx}
\usepackage{authblk}
\usepackage{times}
\usepackage{hyperref}
\usepackage{natbib}
\usepackage{xcolor}
\usepackage{subcaption}
\usepackage{algorithm}
\usepackage{algpseudocode}

\theoremstyle{plain}
\newtheorem{theorem}{Theorem}
\newtheorem{proposition}[theorem]{Proposition}

\theoremstyle{definition}
\newtheorem{definition}{Definition}[section]

\theoremstyle{remark}
\newtheorem{remark}{Remark}
\title{\textbf{Renormalizable Spectral-Shell Dynamics as the Origin of Neural Scaling Laws}}

\author[1]{Yizhou Zhang\footnote{Corresponding Author}}
\affil[1]{zyizhou96@gmail.com}
\date{}

\begin{document}
\maketitle

\begin{abstract}
Neural scaling laws and double-descent phenomena suggest that deep-network
training obeys a simple macroscopic structure despite highly nonlinear
optimization dynamics.
We derive such structure directly from gradient descent in function space.
For mean-squared error loss, the training error evolves as
$\dot e_t=-M(t)e_t$ with $M(t)=J_{\theta(t)}J_{\theta(t)}^{\!*}$, a time-dependent
self-adjoint operator induced by the network Jacobian.
Using Kato perturbation theory, we obtain an exact system of coupled modewise
ODEs in the instantaneous eigenbasis of $M(t)$.

To extract macroscopic behavior, we introduce a logarithmic spectral-shell
coarse-graining and track quadratic error energy across shells.
Microscopic interactions within each shell cancel identically at the energy
level, so shell energies evolve only through dissipation and external
inter-shell interactions.
We formalize this via a \emph{renormalizable shell-dynamics} assumption, under
which cumulative microscopic effects reduce to a controlled net flux across
shell boundaries. This shell-form equation unifies
lazy-regime training and feature learning as two limits of the same
spectral-shell dynamics.

Assuming an effective power-law spectral transport in a relevant resolution
range, the shell dynamics admits a self-similar solution with a moving
resolution frontier and explicit scaling exponents, which explains neural scaling laws and double descent.
\end{abstract}

\section{Introduction}
\label{sec:intro}

Modern deep neural networks are trained by strongly nonlinear and
high-dimensional optimization procedures.
Nevertheless, across architectures, datasets, and training recipes,
their performance exhibits remarkably simple and robust regularities.
Most prominently, empirical studies have shown that training and test losses
obey approximate power laws in model size, dataset size, or total compute
over wide dynamic ranges
\citep{hestness2017deep,kaplan2020scaling,henighan2020scaling,
hoffmann2022training,hernandez2021scaling,wei2022emergent}.
Related law-like behaviors have been observed in precision scaling,
pruning, sparsification, and model densification
\citep{kumar2024scaling,sorscher2022beyond,rosenfeld2021predictability,
blalock2020state,Xiao2024DensingLaw}.
These phenomena strongly suggest the existence of an underlying
macroscopic structure governing how error is redistributed and dissipated
during training.

\paragraph{Limitations of existing theories.}
A classical theoretical explanation is provided by the neural tangent kernel
(NTK) framework
\citep{jacot2018ntk,lee2019wide,bietti2021inductive},
which models training as a linear dynamics in function space with a fixed kernel.
While this perspective explains certain spectral biases and convergence rates,
it relies on a frozen feature map and breaks down once representation learning
becomes significant.
Recent work has therefore emphasized feature-learning dynamics, using
mean-field, tensor-program, or dynamical systems approaches
\citep{yang2021tensor,canatar2022kernel,bordelon2024dynamical,
bordelon2023feature,wang2023spectralevolution}.
These models successfully capture richer behavior beyond the lazy regime,
but typically treat NTK-like and feature-learning regimes via distinct
approximations and do not provide a single operator-level equation that
interpolates continuously between them.

\paragraph{A spectral-shell perspective.}
In this work, we take a function-space and operator-theoretic viewpoint.
For mean-squared error loss, the training error $e_t$ evolves according to the
exact linear equation
\[
\dot e_t = -M(t)e_t,
\qquad
M(t)=J_{\theta(t)}J_{\theta(t)}^{\!*},
\]
where $M(t)$ is a time-dependent, self-adjoint operator induced by the network
Jacobian.
Rather than assuming fixed features or a prescribed kernel, we analyze the
training dynamics directly through the evolving spectral structure of $M(t)$.

Using Kato perturbation theory, we derive an exact system of coupled modewise
ODEs in the instantaneous eigenbasis of $M(t)$.
The resulting dynamics are fully rigorous but highly nonlocal across modes.
To extract macroscopic behavior, we introduce a logarithmic
\emph{spectral-shell coarse-graining} and track the evolution of
\emph{quadratic error energy} within each shell.
A key structural fact is that microscopic interactions within a shell,
arising from eigenbasis drift, are strictly antisymmetric and cancel
identically at the level of quadratic energy.
Consequently, shell energies evolve only through dissipation and
\emph{external interaction effects}---namely, energy exchange with other shells.

\paragraph{Renormalizable shell dynamics.}
This observation leads to an exact shell-level energy balance law.
To close it at macroscopic scales, we introduce a
\emph{renormalizable shell-dynamics assumption}:
after coarse-graining, the cumulative effect of microscopic degrees of freedom
can be absorbed into a controlled net flux across shell boundaries,
with subleading corrections becoming negligible.
This notion of renormalizability is inspired by analogous constructions in
statistical physics and condensed-matter theory, and it does not require
a continuum limit nor impose any \emph{a priori} direction of spectral transport.

When spectral shells are sufficiently dense over a dynamically relevant range,
the shell dynamics admits an continuum approximation in the spectral
variable $\lambda$, yielding an effective transport--dissipation equation.
We emphasize that this PDE is a convenient representation of shell dynamics,
not its fundamental starting point. Throughout the paper, renormalizability and effective power-law transport are treated as macroscopic assumptions, not as consequences of weak coupling or locality.

\paragraph{Effective transport and scaling laws.}
We further assume that, in the relevant resolution range, the renormalized
spectral flux can be summarized by an effective power-law transport form.
Under this single coarse-grained assumption, the shell dynamics admits a
self-similar solution with a moving resolution frontier, geometric amplitude
growth, and power-law dissipation.
This structure yields explicit scaling-law exponents and provides a unified
explanation for neural scaling laws and double-descent phenomena.

Within this framework, lazy (NTK-like) training and feature learning arise as
two limits of the same spectral-shell dynamics.
When the effective transport vanishes, the dynamics reduce to pure dissipation
with fixed features.
When transport is active, spectral mass flows across resolutions, inducing
representation learning.
Thus, both regimes---and the continuum between them---are unified by a single
operator-level shell dynamics governing the redistribution of error during
training.

\section{Preliminaries: Function-Space Dynamics and Error Evolution}
\label{sec:preliminaries}

We consider supervised learning with mean-squared error (MSE) loss in the
Hilbert space \(H=L^2(p)\) induced by the data distribution \(p(x)\).
All functions are identified up to sets of measure zero, and the inner product
is defined by
\[
\langle f,g\rangle := \mathbb{E}_{x\sim p}\,[\,f(x)\,g(x)\,], 
\qquad 
\|f\|_2^2 = \langle f,f\rangle.
\]

\subsection{Neural networks as functions in $L^2(p)$}
A neural network with parameters \(\theta\in\mathbb{R}^N\) represents a 
time-dependent function \(f_{\theta(t)}\in H\).
Given a target function \(f^*\in H\), we define the error at time \(t\) as
\[
e_t := f^* - f_{\theta(t)}.
\]
Throughout this work, we emphasize that no structural assumptions such as fixed
features, invariant kernels, or predetermined mode decompositions are imposed:
the representation \(f_{\theta(t)}\) is allowed to evolve freely in \(H\), and
all structure in the dynamics arises directly from gradient descent.

\subsection{Gradient descent induces a linear evolution in function space}
For MSE loss,
\[
\mathcal{L}(\theta)
= \frac{1}{2}\,\|f^*-f_\theta\|_2^2
= \frac{1}{2}\,\|e\|_2^2,
\]
the gradient flow in parameter space is
\[
\dot\theta(t) = -\nabla_\theta \mathcal{L}(\theta(t)).
\]
By the chain rule,
\[
\dot e_t
= -J_{\theta(t)}\,\dot\theta(t),
\]
where \(J_{\theta}\colon\mathbb{R}^N\to H\) denotes the Jacobian operator
\((J_\theta v)(x) = \nabla_\theta f_\theta(x)\cdot v\).
{\color{red}}
Substituting the gradient flow equation gives
\[
\dot e_t 
= -J_{\theta(t)}\, J_{\theta(t)}^{\!*}\, e_t.
\]
Thus the error evolves according to a \emph{linear} operator in function space:
\[
\boxed{
\dot e_t = -M(t)\,e_t,
\qquad
M(t) := J_{\theta(t)}\,J_{\theta(t)}^{\!*}.
}
\]
The operator \(M(t)\) is self-adjoint and positive semidefinite.
It is the \emph{only} operator that governs the evolution of the error in our
analysis; its time dependence reflects the evolving representation of the
network under gradient descent.

In this section we work in the continuous-time gradient flow limit,
\[
\dot\theta(t) = -\nabla_\theta \mathcal{L}(\theta(t)),
\]
which corresponds to taking the learning rate to be infinitesimal.
A discrete-time update with (possibly time-varying) step sizes
\(\eta_k\),
\[
\theta_{k+1} = \theta_k - \eta_k \nabla_\theta \mathcal{L}(\theta_k),
\]
can be viewed as a time reparameterization of this flow, where the
effective training time is proportional to the accumulated step size
\(\sum_k \eta_k\).
Thus, learning rate schedules do not change the form of the operator
dynamics \(\dot e_t = -M(t)e_t\); instead, they induce a non-uniform
rescaling of time that will later be absorbed into the effective time
variable \(\tau(t)\) in our spectral analysis.

\subsection{Spectral decomposition of \(M(t)\)}
For each fixed \(t\), the operator \(M(t)\) acts on a finite-dimensional subspace
of \(H\) determined by the network's Jacobian, and therefore admits a discrete
spectral decomposition.
We write
\[
M(t)\,\varphi_u(t) = \lambda_u(t)\,\varphi_u(t), 
\qquad u\in\mathcal{U}(t),
\]
where \(\mathcal{U}(t)\) is a finite or countable index set, the eigenvalues
\(\lambda_u(t)\ge 0\), and the eigenfunctions 
\(\{\varphi_u(t)\}_{u\in\mathcal{U}(t)}\) form an orthonormal family in \(H\).
Expanding the error in this moving eigenbasis gives
\[
e_t = \sum_{u\in\mathcal{U}(t)} g_u(t)\,\varphi_u(t).
\]
The amplitudes \(g_u(t)\) encode how much error lies along each instantaneous
mode of \(M(t)\), and the time dependence of both \(\lambda_u(t)\) and
\(\varphi_u(t)\) reflects the evolution of the network's representation.

\subsection{Why a spectral formulation?}
Although \(M(t)\) generally has low rank compared to the ambient dimension of
\(H\), its eigenstructure provides a natural lens on the error dynamics.
The eigenvalues \(\lambda_u(t)\) quantify the rate at which error aligned with
mode \(u\) is dissipated, while the evolution of \(\varphi_u(t)\) captures the
``feature learning'' aspect of training.
In Section~\ref{sec:mode_ode}, we show that the amplitudes \(g_u(t)\) satisfy a
rigorous, coupled system of mode ODEs derived from Kato's perturbation theory.
In Section~\ref{sec:mode_to_pde}, we then coarse-grain these discrete modes to
obtain a continuous spectral PDE that describes the statistical redistribution
of error across many nearby modes.

\section{From Modewise ODEs to Coarse-Grained Spectral Dynamics}
\label{sec:mode_to_pde}

In this section we derive the fundamental \emph{coarse-grained spectral structure}
governing the dynamics of the error under gradient descent.
Starting from the exact function-space evolution $\dot e_t=-M(t)e_t$,
we first express the dynamics in the instantaneous eigenbasis of $M(t)$,
obtaining a fully rigorous system of coupled modewise ODEs.

Rather than attempting to track individual modes, which is intractable at
macroscopic scales, we then introduce a logarithmic spectral-shell
coarse-graining and study the evolution of the \emph{quadratic error energy}
carried by each shell.
This leads to an exact shell-level balance law, in which shell-internal coupling
cancels identically and all nontrivial interactions appear as inter-shell energy
fluxes.

When the shell spacing is sufficiently fine over a dynamically relevant range,
this shell dynamics admits an continuum approximation in the spectral
variable $\lambda$, yielding a transport–dissipation PDE.
However, the shell-level formulation itself is primary and does not rely on any
continuum limit. In particular, we never assume that the spectrum becomes continuous or that a limit as the shell spacing vanishes exists; all continuum expressions should be read as local approximations to discrete shell differences.

\subsection{Exact Mode ODEs in a Drifting Eigenbasis}
\label{sec:mode_ode}

For each time \(t\), the operator \(M(t)=J_{\theta(t)}J_{\theta(t)}^{\!*}\) is
self-adjoint and positive semidefinite, acting on a finite-dimensional subspace
of \(H=L^2(p)\).
Thus its spectrum is discrete.
Let
\[
M(t)\varphi_u(t) = \lambda_u(t)\,\varphi_u(t), \qquad 
u \in \mathcal{U}(t),
\]
denote an orthonormal eigenbasis of \(M(t)\).
Expanding the error in this moving basis gives
\[
e_t = \sum_{u} g_u(t)\,\varphi_u(t),
\qquad
g_u(t) = \langle e_t, \varphi_u(t)\rangle.
\]

Differentiating in time yields
\[
\partial_t g_u(t) 
= \langle \partial_t e_t,\, \varphi_u(t)\rangle
+ \big\langle e_t,\,\partial_t \varphi_u(t)\big\rangle.
\]
Using \(\dot e_t = -M(t)e_t\), the first term becomes
\[
\langle -M(t)e_t,\, \varphi_u\rangle
= -\lambda_u(t) g_u(t).
\]
The second term encodes the rotation of the eigenbasis.
From Kato's perturbation theory \citep{kato2012short,Zwiebach_Adiabatic} for differential self-adjoint operators,
the evolution of the eigenfunctions satisfies
\[
\langle \partial_t \varphi_v(t),\,\varphi_u(t)\rangle
= 
\begin{cases}
\displaystyle
\frac{
\langle \varphi_u(t),\,\dot M(t)\,\varphi_v(t)\rangle
}{
\lambda_v(t)-\lambda_u(t)
},
& v\neq u,\\[1em]
0,& v=u,
\end{cases}
\]
where \(\dot M(t)\) is the operator derivative of \(M(t)\).
Therefore
\[
\big\langle e_t,\,\partial_t \varphi_u(t)\big\rangle
= 
\sum_{v\neq u}
g_v(t)\,
\frac{
\langle \varphi_u(t),\,\dot M(t)\,\varphi_v(t)\rangle
}{
\lambda_v(t)-\lambda_u(t)
}.
\]

Combining both contributions, we obtain the exact coupled mode ODE:
\[
\boxed{
\partial_t g_u(t)
+
\sum_{v\neq u}
g_v(t)\,\Omega_{v\to u}(t)
=
-\lambda_u(t)\,g_u(t),
}
\]
where the coupling coefficients are
\[
\Omega_{v\to u}(t)
=
\frac{
\langle \varphi_u(t),\,\dot M(t)\,\varphi_v(t)\rangle
}{
\lambda_v(t) - \lambda_u(t)
}.
\]

This ODE system is fully rigorous and contains all aspects of the dynamics:
local dissipation \(-\lambda_u g_u\), nonlocal mode coupling
\(\sum_{v\neq u} g_v\,\Omega_{v\to u}\), and the drift of the feature basis
through the time dependence of the eigenfunctions \(\varphi_u(t)\).

\subsection{Logarithmic Spectral Shells and Shell Energies}
\label{sec:shell_def}

The exact modewise dynamics derived in Section~\ref{sec:mode_ode} are fully
rigorous but too fine-grained for macroscopic analysis.
To expose the coarse structure, we group modes into logarithmic spectral shells
and track the \emph{quadratic} error energy carried by each shell.

Fix a ratio $q>1$ and define a logarithmic partition of the positive spectrum:
\[
\lambda_\alpha := \lambda_0\, q^\alpha,
\qquad
S_\alpha := \bigl\{\,u:\ \lambda_u(t)\in[\lambda_\alpha,\lambda_{\alpha+1})\,\bigr\},
\qquad \alpha\in\mathbb{Z}.
\]
We also define cumulative shells $S_{\le\alpha}:=\bigcup_{\gamma\le\alpha}S_\gamma$.

For each shell, we define the \emph{shell quadratic energy}
\[
E_\alpha(t)
:= \frac{1}{2}\sum_{u\in S_\alpha} g_u(t)^2,
\qquad
E_{\le\alpha}(t)
:= \sum_{\gamma\le\alpha} E_\gamma(t).
\]
This choice is canonical: the full function-space loss is
$\mathcal L(t)=\frac{1}{2}\|e_t\|_2^2=\frac{1}{2}\sum_u g_u(t)^2$,
so $\{E_\alpha\}$ provides an exact energy bookkeeping across spectral shells. 

\subsection{Exact Cancellation of Shell-Internal Coupling for Quadratic Energy}
\label{sec:shell_internal}

Recall the exact mode ODE
\[
\partial_t g_u(t)
+
\sum_{v\neq u} g_v(t)\,\Omega_{v\to u}(t)
=
-\lambda_u(t)\,g_u(t),
\qquad
\Omega_{v\to u}(t)
=
\frac{\langle \varphi_u(t),\dot M(t)\varphi_v(t)\rangle}{\lambda_v(t)-\lambda_u(t)}.
\]
Since $\dot M(t)$ is self-adjoint and the denominator is antisymmetric, we have
the strict antisymmetry
\[
\Omega_{v\to u}(t) = -\Omega_{u\to v}(t),
\qquad u\neq v.
\]

\begin{proposition}[Shell-internal cancellation (quadratic energy)]
\label{prop:shell_internal}
For any spectral shell $S_\alpha$, the shell-internal coupling does not change
the quadratic shell energy:
\[
\sum_{u\in S_\alpha}\sum_{v\in S_\alpha,\ v\neq u}
g_v(t)\,g_u(t)\,\Omega_{v\to u}(t)
=
0.
\]
\end{proposition}

\begin{proof}
Pair terms $(u,v)$ and $(v,u)$ inside the double sum. Using
$\Omega_{v\to u}=-\Omega_{u\to v}$, we get
\[
g_v g_u \Omega_{v\to u} + g_u g_v \Omega_{u\to v}
=
g_u g_v\bigl(\Omega_{v\to u}+\Omega_{u\to v}\bigr)=0.
\]
Summing over all unordered pairs in $S_\alpha$ yields the claim.
\end{proof}

\begin{remark}[Interpretation]
Proposition~\ref{prop:shell_internal} formalizes the key point:
\emph{shell-internal coupling is a pure redistribution (rotation) of error
among modes within the shell}. It can change the individual $g_u$'s, but it
cannot change $\sum_{u\in S_\alpha} g_u^2$.
Therefore, any change of the shell energy $E_\alpha(t)$ must come from
(i) local dissipation and (ii) energy exchange with other shells.
\end{remark}

\subsection{Exact Shell Balance Law for Quadratic Energy}
\label{sec:shell_balance}

Differentiate $E_\alpha(t)=\frac{1}{2}\sum_{u\in S_\alpha} g_u^2$:
\[
\frac{d}{dt}E_\alpha(t)
=
\sum_{u\in S_\alpha} g_u(t)\,\partial_t g_u(t).
\]
Substituting the mode ODE yields
\begin{align}
\frac{d}{dt}E_\alpha(t)
&=
-\sum_{u\in S_\alpha}\lambda_u(t)\,g_u(t)^2
-\sum_{u\in S_\alpha}\sum_{v\neq u} g_v(t)\,g_u(t)\,\Omega_{v\to u}(t).
\label{eq:Ealpha_raw}
\end{align}
Split the coupling sum into $v\in S_\alpha$ and $v\notin S_\alpha$.
By Proposition~\ref{prop:shell_internal}, the $v\in S_\alpha$ contribution cancels.
Hence we obtain the exact shell energy balance:
\begin{equation}
\label{eq:shell_energy_balance}
\boxed{
\frac{d}{dt}E_\alpha(t)
=
-\sum_{u\in S_\alpha}\lambda_u(t)\,g_u(t)^2
-\sum_{\beta\neq\alpha}\ 
\sum_{\substack{u\in S_\alpha\\ v\in S_\beta}}
g_v(t)\,g_u(t)\,\Omega_{v\to u}(t).
}
\end{equation}

The first term is \emph{pure shell dissipation}.
The second term is \emph{pure inter-shell exchange}.

\subsection{Inter-Shell Fluxes and Discrete Conservation Structure}
\label{sec:shell_flux}

Define the inter-shell quadratic-energy flux from $S_\beta$ to $S_\alpha$ by
\[
\mathcal F_{\beta\to\alpha}(t)
:=
-\sum_{\substack{u\in S_\alpha\\ v\in S_\beta}}
g_v(t)\,g_u(t)\,\Omega_{v\to u}(t).
\]
Then equation~\eqref{eq:shell_energy_balance} becomes
\begin{equation}
\label{eq:shell_energy_balance_flux}
\boxed{
\frac{d}{dt}E_\alpha(t)
=
-\sum_{u\in S_\alpha}\lambda_u(t)\,g_u(t)^2
+
\sum_{\beta\neq\alpha}\mathcal F_{\beta\to\alpha}(t).
}
\end{equation}

Moreover, antisymmetry implies a strict action--reaction identity:
\[
\mathcal F_{\beta\to\alpha}(t)
=
-\mathcal F_{\alpha\to\beta}(t).
\]
Therefore, coupling conserves total quadratic energy across shells:
\[
\sum_\alpha \sum_{\beta\neq\alpha}\mathcal F_{\beta\to\alpha}(t)=0,
\]
and the only mechanism that decreases $\sum_\alpha E_\alpha(t)=\frac{1}{2}\|e_t\|_2^2$
is the dissipation term $-\sum_u \lambda_u g_u^2$.

\subsection{From Exact Shell Balance to Renormalized Dynamics}
\label{sec:shell_to_renorm}

The balance law~\eqref{eq:shell_energy_balance_flux} is an exact conservation-type
bookkeeping equation across logarithmic shells, requiring no continuum limit
and no PDE interpretation.
However, the fluxes $\mathcal F_{\beta\to\alpha}(t)$ remain nonlocal and depend
on microscopic details of the evolving operator.

In the remainder of the paper, we do not attempt to characterize
$\mathcal F_{\beta\to\alpha}$ microscopically.
Instead, Section~4 introduces a \emph{renormalizable shell-dynamics assumption}
that closes the cumulative effect of inter-shell fluxes at the shell level.
All subsequent results rely exclusively on this shell-level interface.

We emphasize that all subsequent results in this paper rely exclusively on
this shell-level energy balance and the renormalizable flux interface introduced
in Section~\ref{sec:renorm_shell}.
The continuum PDE description is used only as a convenient approximation
when shell resolution permits.

\section{Renormalizable Shell Dynamics and Effective Power-Law Transport}
\label{sec:renorm_shell}

Section~\ref{sec:mode_to_pde} provides an exact modewise ODE system and an effective spectral PDE description.
In this section we introduce an explicitly \emph{coarse-grained} interface for macroscopic analysis: logarithmic spectral shells, a renormalizable flux bookkeeping condition across shells, and a single effective transport assumption used in the remainder of the paper.
We emphasize that this section is not a microscopic characterization of all possible operators $M(t)$; rather, it formulates the minimal coarse-grained structure needed to derive self-similar spectral dynamics and scaling laws.


\subsection{Log-shell partition and quadratic shell energies}

Fix a ratio $q>1$ and define a logarithmic partition $\{\lambda_\alpha\}_{\alpha\in\mathbb{Z}}$ by
\[
\lambda_\alpha := \lambda_0 q^\alpha,
\qquad
S_\alpha := \{u:\ \lambda_u(t)\in[\lambda_\alpha,\lambda_{\alpha+1})\},
\qquad
S_{\le\alpha}:=\bigcup_{\gamma\le\alpha}S_\gamma.
\]
We work with the \emph{quadratic} shell energies introduced in Section~\ref{sec:shell_def}:
\[
E_\alpha(t):=\frac12\sum_{u\in S_\alpha} g_u(t)^2,
\qquad
E_{\le\alpha}(t):=\sum_{\gamma\le\alpha}E_\gamma(t).
\]
This choice is canonical because $\mathcal L(t)=\frac12\|e_t\|_2^2=\sum_\alpha E_\alpha(t)$.

\subsection{Exact shell bookkeeping and the role of dissipation}

Section~\ref{sec:shell_balance} already established the exact identity
\[
\frac{d}{dt}E_\alpha(t)
=
-D_\alpha(t) + \sum_{\beta\neq\alpha}\mathcal F_{\beta\to\alpha}(t),
\qquad
D_\alpha(t):=\sum_{u\in S_\alpha}\lambda_u(t)\,g_u(t)^2,
\]
where $\mathcal F_{\beta\to\alpha}(t)$ is the inter-shell quadratic-energy flux and satisfies
$\mathcal F_{\beta\to\alpha}(t)=-\mathcal F_{\alpha\to\beta}(t)$.
In particular, \emph{shell-internal coupling does not change $E_\alpha$} and the \emph{only} mechanism that decreases
$\sum_\alpha E_\alpha(t)$ is dissipation $\sum_\alpha D_\alpha(t)$.

\subsection{Definition: renormalizable shell dynamics via cumulative \emph{energy} flux}

The notion of renormalizability adopted here follows the standard usage in
statistical physics and field theory.
Rather than requiring microscopic locality or exact continuum limits, external interaction effects on a resolution shell are integrated out and absorbed into effective inter-shell fluxes, while subleading corrections become irrelevant at coarse scales.
This philosophy underlies Wilsonian renormalization group theory
\citep{wilson1983renormalization,kadanoff1966scaling}, shell models of turbulence and energy cascades
\citep{kolmogorov1995turbulence}, and effective hydrodynamic descriptions of
nonequilibrium systems \citep{forster1977large,spohn2012large}.
Our definition formalizes this principle at the level of spectral-shell energy
dynamics.

\begin{definition}[Renormalizable spectral-shell dynamics (energy form)]
\label{def:renorm_energy}
The modewise dynamics are said to be \emph{(weakly) renormalizable} with respect to the log-shell partition if,
for sufficiently large $\alpha$, the shell energies admit a closed balance of the form
\[
\label{eq:sec5_renorm_shell}
\frac{d}{dt}E_\alpha(t)
=
-D_\alpha(t)
+\mathcal{F}_\alpha^{\rm (net)}(t)
+\mathcal{R}_\alpha(t),
\]
where the coarse-grained interaction contribution admits a net-flux form
\[
\mathcal{F}_\alpha^{\rm (net)}(t)=J_{\le\alpha}(t)-J_{\le\alpha-1}(t),
\]
and the remainder is negligible compared to the leading dissipation:
\[
|\mathcal{R}_\alpha(t)| \ll D_\alpha(t)\qquad \text{as }\alpha\to\infty.
\]
Here $J_{\le\alpha}(t)$ is the cumulative \emph{quadratic-energy} flux across the shell boundary at $\alpha$,
defined by the exact microscopic inter-shell fluxes as
\[
\boxed{
J_{\le\alpha}(t)
:=
\sum_{\beta>\alpha}\ \mathcal F_{\beta\to(\le\alpha)}(t)
=
\sum_{\beta>\alpha}\ \sum_{\gamma\le\alpha}\mathcal F_{\beta\to\gamma}(t).
}
\]
Moreover, the cumulative flux is \emph{integrable at coarse scales} in the sense that for each shell $\alpha$ there exists $T_\alpha>0$ such that for all $T>T_\alpha$,
\[
\int_0^T |J_{\le\alpha}(t)|\,dt
\;\le\;
C \int_0^T D_\alpha(t)\,dt .
\]

\end{definition}

\begin{remark}[What the cumulative flux represents]
Definition~\ref{def:renorm_energy} does not assume microscopic locality of coupling.
All shells $\beta>\alpha$ are included in $J_{\le\alpha}$.
Renormalizability asserts only that after coarse-graining the net inter-shell exchange can be summarized
by a controlled boundary flux plus a negligible remainder.
\end{remark}

\begin{remark}[No direction is assumed]
Renormalizability imposes no restriction on the sign or direction of $J_{\le\alpha}(t)$.
All later results depend only on the existence of a controlled net-flux representation,
not on an \emph{a priori} cascade direction.
\end{remark}

\subsection{PDE-approximability across shells (energy continuum)}
When the shell spacing is sufficiently fine over a resolution range, we may introduce an energy density
$\varepsilon(\lambda,t)$ such that
\[
E_\alpha(t)\approx \int_{\lambda_\alpha}^{\lambda_{\alpha+1}}\varepsilon(\lambda,t)\,d\lambda,
\qquad
J_{\le\alpha}(t)-J_{\le\alpha-1}(t)\approx \int_{\lambda_\alpha}^{\lambda_{\alpha+1}}J(\lambda,t)\,d\lambda,
\]
This is a \emph{local} continuum approximation over that range; we do not require a global continuous-spectrum limit. As in nonequilibrium statistical physics \cite{forster1977large,spohn2012large}, we do not assume the existence of an exact continuum limit; rather, renormalizability refers to the existence of a closed effective description after coarse-graining. We stress again that our continuum equations do not assume the existence of a genuine continuous-spectrum limit. Under this approximation we acquire the expression of the total loss function:
\[
\mathcal L(t)\approx\int_{\lambda_{\min}}^{\lambda_{\max}} \varepsilon(\lambda,t)\,d\lambda.
\]
They represent a local approximation of the discrete shell-flux difference when shells are sufficiently dense over a dynamically relevant resolution range. Formally, this amounts to approximating the discrete net-flux difference $J_{\le\alpha}(t)-J_{\le\alpha-1}(t)$
by a first-order finite-difference representation of $\partial_\lambda J$ over a dense but finite shell range. This energy continumm yields, to leading order, the energy transport--dissipation equation
\begin{equation}
\label{eq:sec5_energy_pde}
\partial_t \varepsilon(\lambda,t)
+\partial_\lambda J(\lambda,t)
=
-2\lambda\,\varepsilon(\lambda,t),
\end{equation}
on the relevant resolution range.
The factor $2$ reflects the fact that dissipation in the mode ODE is $-\lambda_u g_u$,
hence quadratic energy dissipates as $-\frac{d}{dt}\frac12 g_u^2=\lambda_u g_u^2$.

\subsection{RMS amplitude notation}

For compatibility with earlier GSD/GRSD forms, we also define an RMS amplitude density
\begin{equation}
\label{eq:sec5_rms_def}
g(\lambda,t):=\sqrt{2\,\varepsilon(\lambda,t)}.
\end{equation}
Then
\[
\mathcal L(t)\approx\int_{\lambda_{\min}}^{\lambda_{\max}} \varepsilon(\lambda,t)\,d\lambda
=\int_{\lambda_{\min}}^{\lambda_{\max}} g(\lambda,t)^2\,d\lambda,
\]
and equation~\eqref{eq:sec5_energy_pde} can be viewed as the energy-level backbone behind the GRSD tail forms written in terms of $g$.

\subsection{Unified view of NTK and feature learning through spectral transport}

The shell-energy formulation provides a clean unification of lazy (NTK-like)
training and feature learning through a single macroscopic quantity:
the spectral drift velocity $v(\lambda,t)$. Recall the definition
\[
J(\lambda,t)=v(\lambda,t)\,g(\lambda,t),
\]
which defines the effective transport velocity $v(\lambda,t)$ in the coarse-grained description.

\paragraph{Lazy / NTK regime ($v\equiv 0$).}
If the kernel operator is effectively frozen during training ($\dot M(t)\approx 0$),
then $v(\lambda,t)\equiv 0$ and no inter-shell transport occurs.
The shell dynamics reduce to
\[
\frac{d}{dt}E_\alpha(t) = -D_\alpha(t),
\]
so each shell decays independently under dissipation.
No resolution frontier forms, and the dynamics coincide with classical NTK theory.

\paragraph{Feature learning regime ($v\neq 0$).}
When the representation evolves ($\dot M(t)\neq 0$), the drift velocity becomes
active and induces inter-shell energy transport.
Energy is redistributed across resolutions before being dissipated,
producing a moving resolution frontier and the GRSD tail described in
Sections~5--6.

Crucially, this redistribution preserves total quadratic energy and does not
increase training loss.
Its observable consequences depend on how the transported energy aligns with
the evaluation distribution (training vs.\ test).

\paragraph{Continuum of regimes.}
Between these extremes lies a continuum:
\begin{itemize}
\item If $|v(\lambda,t)|\ll \lambda$, dissipation dominates and training is
effectively NTK-like.
\item If $|v(\lambda,t)|$ is comparable to or larger than $\lambda$ over a
resolution range, transport reshapes the spectrum and feature learning emerges.
\end{itemize}

Thus, lazy training and feature learning are not distinct dynamical theories,
but limiting behaviors of the same renormalized shell-energy dynamics.

\paragraph{Interpretation.}
From this perspective, feature learning corresponds to \emph{spectral energy
transport induced by representation drift}, while NTK corresponds to the
degenerate zero-transport limit.
Both are unified within the same operator-theoretic framework.

\subsection{Summary}

The key structural takeaway is:

\begin{quote}
\emph{Shell-internal coupling is a rotation that preserves quadratic energy within each shell.
Any macroscopic evolution of shell energies is fully captured by (i) dissipation and (ii) inter-shell boundary flux.}
\end{quote}

All scaling-law consequences in later sections are derived from the renormalized boundary-flux interface
\eqref{eq:sec5_renorm_shell}--\eqref{eq:sec5_energy_pde} in a high-resolution range.

\section{Scale-Free GRSD Under Power-Law–Compatible Transport}
\label{sec:self_similar_grsd}

In this section we analyze the effective GRSD regime arising in an intermediate
scale-free training window. Empirically, modern deep models exhibit extended
intervals $t_0<t<T$ during which macroscopic spectral observables evolve as
straight lines in log--log coordinates, indicating the absence of intrinsic
scales.

Within such a window, the effective spectral transport must itself be
approximately scale-free. While the most general scale-free drift takes the
form $v(\lambda,t)=\phi(\lambda/t)$, requiring \emph{pure power-law temporal
scaling} of observables rules out generic $\phi$ and selects the linear
scale-free transport
\begin{equation}
\label{eq:sec6_v_c0}
v(\lambda,t)=c_0\,\frac{\lambda}{t},
\qquad t_0<t<T,
\end{equation}
up to subleading corrections.
In the remainder of this section we take \eqref{eq:sec6_v_c0} as an effective
model assumption and solve the resulting GRSD dynamics explicitly.

\paragraph{Why scale-free observables rigidly constrain the drift.}
We emphasize that the power-law form
\[
\label{sec5_characteristics}
v(\lambda,t)=c_0\,\frac{\lambda}{t}
\]
is not introduced as an arbitrary modeling choice.
Rather, it is rigidly selected by the requirement that macroscopic spectral observables
remain scale-free over an extended training window.

Empirically, in this regime, shell-resolved quantities such as the spectral energy density,
cumulative shell energies, and total loss evolve as straight lines in log--log coordinates
over multiple decades of time and resolution.
This behavior implies not merely scale invariance at a fixed time, but a joint
scale covariance between the spectral coordinate $\lambda$ and the training time $t$.

At the level of the coarse-grained shell dynamics, this covariance severely restricts
the admissible form of the effective transport velocity.
While the most general scale-free drift compatible with dimensional analysis can be written as
$v(\lambda,t)=\phi(\lambda/t)$, generic choices of $\phi$ induce nonlinear reparameterizations
of the spectral coordinate and therefore produce curvature in log--log observables.
Requiring that these observables remain asymptotically affine in log--log space
selects the linear form $\phi(u)=c_0 u$, yielding $v(\lambda,t)=c_0\lambda/t$
up to subleading corrections.

\paragraph{On the origin of the scale-free drift form.}
The specific choice
\(
v(\lambda,t)=c_0\,\lambda/t
\)
should not be interpreted as an ad hoc modeling assumption.
Under a set of strong but practically natural conditions—such as controlled
training trajectories, uniformly bounded gradients, and zero-mean random
initialization—the same scale-free drift form can be derived directly from the
training dynamics.
A concrete derivation under such conditions is given in
\citet{zhang2026doeslearningrenormalizesufficient}.
In the present work, we do not rely on these sufficient conditions and instead
take Eq. \eqref{sec5_characteristics} as a macroscopic consequence of scale-free spectral evolution;
the aforementioned derivation serves only to demonstrate that the assumed drift
is dynamically realizable rather than postulated.

\paragraph{Lagrangian spectral coordinates as the renormalized frame.}
The emergence of the Lagrangian spectral coordinate $u$
should not be viewed as a technical device introduced solely to solve the transport equation.
Instead, it represents the natural renormalized variable induced by the shell-level
flux interface.

At the discrete shell level, renormalizability asserts that the cumulative effect of
inter-shell interactions can be summarized by a controlled boundary flux.
In the continuum approximation, this structure manifests as a transport term that can
be absorbed into a time-dependent reparameterization of the spectral coordinate.
The characteristic flow Eq. \eqref{sec5_characteristics} implements precisely this absorption.

In the co-moving coordinate $u$, the dynamics reduce to pure dissipation.
All nontrivial macroscopic structure---including scaling laws and apparent emergence---
is therefore encoded in how this co-moving frame maps back to stationary spectral coordinates.
The simplicity of the resulting evolution is not accidental, but a direct consequence
of the renormalizable shell-dynamics assumption.

\subsection{Characteristic flow and scale-free reparameterization}

Starting from the leading-order GRSD equation
\begin{equation}
\partial_t \varepsilon(\lambda,t)
+\partial_\lambda\!\big(v(\lambda,t)\varepsilon(\lambda,t)\big)
=-2\lambda\,\varepsilon(\lambda,t),
\label{eq:sec6_grsd_pde}
\end{equation}
with drift \eqref{eq:sec6_v_c0}, the characteristic curves satisfy
\begin{equation}
\frac{d\lambda}{dt}
=c_0\,\frac{\lambda}{t}.
\label{eq:sec6_char}
\end{equation}
Integrating yields the exact scale-free flow
\begin{equation}
\lambda(t;u)
=u\Big(\frac{t}{t_0}\Big)^{c_0},
\qquad
u=\lambda\Big(\frac{t}{t_0}\Big)^{-c_0},
\label{eq:sec6_char_solution}
\end{equation}
where $u$ defines a stationary (Lagrangian) spectral coordinate.
This flow is strictly power-law and affine in log--log coordinates.

\subsection{Energy density along characteristics}
Along a characteristic $\lambda(t;u)$, the energy density evolves as
\begin{equation}
\frac{d}{dt}\log\varepsilon(\lambda(t),t)
=-2\lambda(t).
\end{equation}
Substituting \eqref{eq:sec6_char_solution} and integrating gives
\begin{equation}
\varepsilon(u,t)
=
\varepsilon_0(u)\,
\exp\!\Big[
-\frac{2u}{1+c_0}
\Big(
\Big(\frac{t}{t_0}\Big)^{1+c_0}-1
\Big)
\Big],
\qquad c_0\neq -1.
\label{eq:sec6_eps_lagrangian}
\end{equation}
which is aligned with the spectral evolution conjecture proposed in \cite{zhang2025generalized}:

\begin{equation}
    f^*-f_\lambda(t) = w_\lambda\exp(g(\lambda,t))
\end{equation}
where $g(\lambda,t)\propto  \lambda^{a(\beta)} t^{b(\beta)}$ and thus $w_\lambda^2$ correspond to the starting error on mode with eigenvalue of $\lambda$. This conjecture is helpful for understanding the scaling law for both loss evolution and model compression. It further predicted the existence of learning frontier and defines model density, which is observed to evolve according to \cite{Xiao2024DensingLaw}.  

Rewriting \eqref{eq:sec6_eps_lagrangian} in Eulerian coordinates using
$u=\lambda(t/t_0)^{-c_0}$ yields
\begin{equation}
\boxed{
\varepsilon(\lambda,t)
=
\varepsilon_0\!\big(\lambda\,t^{-c_0}\big)\,
\exp\!\Big(-\frac{2\lambda}{1+c_0}\,t\Big)
}.
\label{eq:sec6_eps_grsd}
\end{equation}

The exponential factor represents physical dissipation, while the prefactor
encodes pure scale-free transport. The competition between transport and
dissipation is governed entirely by the combination $1+c_0$.
This expression constitutes the leading-order GRSD solution in the scale-free
window.

At the level of logarithmic spectral shells, the high-resolution behavior of the
quadratic error energy can be expressed in a factorized form
\begin{equation}
E_\alpha(t)
\;\approx\;
A(\lambda_\alpha\,t^{-c_0})\,
\exp\!\big(-\kappa(t)\,\lambda_\alpha\big),
\qquad
\lambda_\alpha \ \text{beyond the resolution frontier},
\label{eq:shell_energy_factorized}
\end{equation}
where $\lambda_\alpha$ denotes the representative eigenvalue of shell $\alpha$,
$\kappa(t)$ is an effective dissipation scale, and
$A(\lambda,t)$ is a coarse-grained amplitude capturing the combined effects of
spectral transport and initialization.

\paragraph{Remark (Amplitude structure).}
The function $A(\lambda,t)$ is not assumed to take any specific form in general.
\emph{Particularly}, when the initial error spectrum is approximately scale-free
over a range of resolutions, the amplitude may inherit a corresponding power-law
dependence on $\lambda$:
\begin{equation}
\varepsilon(\lambda,t)
=
\big(\lambda\,t^{-c_0}\big)^p\,
\exp\!\Big(-\frac{2\lambda}{1+c_0}\,t\Big)
 = t^{-c_0p}\big(\lambda\big)^p\,
\exp\!\Big(-\frac{2\lambda}{1+c_0}\,t\Big)
 .
\label{eq:sec6_eps_grsd_power_law}
\end{equation}
which leads to a shell-energy profile of the representative form
\begin{equation}
E_\alpha(t)
\;\sim\;
C(t)\lambda_\alpha^{\,p+1}\,
\exp\!\big(-\kappa(t)\,\lambda_\alpha\big),
\qquad
\lambda_\alpha \ \text{beyond the frontier},
\label{eq:shell_energy_powerlaw_example}
\end{equation}
for some exponent $p$.
We emphasize that this expression is provided only as an illustrative example of how
scale-free initialization can manifest at the shell level; no assumption on the sign
or universality of $p$ is made, and our empirical analysis does not rely on this
specific functional form.

The total quadratic loss is
\begin{equation}
\mathcal{L}(t)=\int_{\lambda_{\min}}^{\lambda_{\max}}\varepsilon(\lambda,t) d\lambda.
\approx \int_{\lambda_{\min}}^\infty\varepsilon(\lambda,t) d\lambda.
\end{equation}
where the approximation holds for an exponential decay as lambda increases.
Starting from
\begin{equation}
\varepsilon(\lambda,t)
= (\lambda t^{-c_0})^{p}
\exp\!\Big(-\frac{2\lambda}{1+c_0}\,t\Big)
= t^{-c_0 p}\,\lambda^{p}\,e^{-a t\lambda},
\qquad a:=\frac{2}{1+c_0},
\end{equation}
the total quadratic error is
\begin{equation}
L(t)
\approx \int_{\lambda_{\min}}^\infty \varepsilon(\lambda,t)\,d\lambda
= t^{-c_0 p}\int_{\lambda_{\min}}^\infty \lambda^{p} e^{-a t\lambda}\,d\lambda .
\end{equation}
Performing the change of variables $u = a t \lambda$, we obtain
\begin{equation}
L(t)
= t^{-c_0 p}\,(a t)^{-(p+1)} 
\int_{a t \lambda_{\min}}^\infty u^{p} e^{-u}\,du = t^{-c_0 p}\,(a t)^{-(p+1)} \Gamma(1+p, a t \lambda_{\min}).
\end{equation}
In our theoretic analysis, the time $t\in (t_0,T)$ is in the scale-free range. For any $T$ such that $T \lambda_{\min} \ll \frac{1}{a}$, the remaining integral converges to a constant as $t$ varies, yielding the asymptotic scaling
\begin{equation}
L(t)
\;\propto\;
t^{-c_0 p}\, t^{-(p+1)}
=
t^{-\big[(1+c_0)p+1\big]} .
\end{equation}

\paragraph{Transport--dissipation competition and spectral front formation.}
The explicit solution reveals a universal structural feature:
the dynamics are governed by a competition between scale-free transport and
scale-dependent dissipation.
Transport alone redistributes error mass across resolutions without loss,
while dissipation selectively suppresses high-$\lambda$ components.

Their interplay generates a moving spectral front separating a learned region
from an unresolved tail.
Importantly, this front is not an artifact of a particular initialization,
architecture, or optimizer choice.
It is a robust consequence of renormalizable spectral transport under
power-law--compatible drift.

In the next section, we show that when training and evaluation are performed
on finite samples drawn from the same underlying distribution,
this moving front interacts with unavoidable high-resolution spectral mismatch.
The resulting transient amplification in certain spectral bands provides
a natural and entirely in-domain explanation of the double-descent phenomenon.

\section{Double Descent from In-Domain Finite-Sample Spectral Mismatch}
\label{subsec:double_descent_mismatch}
We now explain the origin of the double-descent phenomenon from the perspective
of renormalizable spectral-shell dynamics.
Crucially, the mechanism described here is entirely \emph{in-domain}:
training and test data are drawn independently from the same underlying
data-generating distribution.
The non-monotonicity of the test loss arises not from distribution shift, but
from unavoidable finite-sample spectral mismatch concentrated in high-resolution
modes.

\subsection{Training dynamics and monotone training loss.}
Throughout training, the error evolves according to
\[
\dot e_t=-M(t)e_t,
\qquad
M(t)=J_{\theta(t)}J_{\theta(t)}^{\!*}\succeq 0.
\]
Measured in the training metric, the quadratic loss
\[
\mathcal L_{\rm tr}(t)
:=\frac12\|e_t\|_{L^2(p)}^2
=\sum_\alpha E_\alpha(t),
\qquad
E_\alpha(t):=\frac12\sum_{u\in S_\alpha} g_u(t)^2,
\]
is strictly non-increasing:
\[
\frac{d}{dt}\mathcal L_{\rm tr}(t)
=-\langle e_t,M(t)e_t\rangle\le 0.
\]
This monotonicity is exact and does not rely on any coarse-graining,
continuum approximation, or weak-coupling assumption.

\subsection{In-domain train/test mismatch as a finite-sample effect.}
Let $P_{\rm tr}$ and $P_{\rm te}$ denote the empirical quadratic forms induced
by the training and test samples, respectively.
Both are unbiased estimators of the same population operator associated with
the underlying distribution, but differ due to finite sampling.
Define the mismatch operator
\[
\Delta := P_{\rm te}-P_{\rm tr}.
\]
The test loss can be written as
\[
\mathcal L_{\rm te}(t)
=\frac12\langle e_t,P_{\rm te}e_t\rangle
=\mathcal L_{\rm tr}(t)
+\frac12\langle e_t,\Delta e_t\rangle.
\]
The second term captures purely finite-sample effects and vanishes only in the
infinite-data limit.

For continuous data distributions or infinite-rank kernel operators, the
population spectrum exhibits vanishing eigenvalue gaps in the spectral tail.
Classical results on empirical covariance and integral operators show that, in
this regime, the associated empirical eigenspaces exhibit $O(1)$ fluctuations
under finite sampling
\citep{koltchinskii2017normal,rosasco2010learning}.
Consequently, even under in-domain sampling, $\Delta$ is generically supported
on high-resolution (spectral-tail) components and is not co-diagonalizable with
the training-induced operator $M(t)$.

\paragraph{Double descent as a generic finite-sample effect.}
We stress that the non-monotonic behavior described here does not rely on
distribution shift, adversarial sampling, or overparameterization per se.
It arises generically whenever the population operator exhibits a slowly decaying
spectral tail and learning proceeds through a transport-dominated regime.

From this perspective, double descent is not a pathological deviation from
monotone learning, but a predictable finite-sample manifestation of
renormalizable spectral transport.
As model capacity or dataset size increases, the location of the
mismatch-dominated spectral region shifts,
modifying the timing and magnitude of the transient amplification
without eliminating the underlying mechanism.

\subsection{Transport-induced amplitude reweighting}
\label{subsec:amplitude_reweighting}

Starting from the GRSD solution \eqref{eq:sec6_eps_grsd}, the spectral error
density admits the exact factorization
\begin{equation}
\varepsilon(\lambda,t)
=
A(\lambda,t)\,
\exp\!\Big(-\frac{2\lambda}{1+c_0}\,t\Big),
\qquad
A(\lambda,t):=\varepsilon_0(\lambda t^{-c_0}).
\label{eq:sec7_amp_def}
\end{equation}
The exponential factor represents pure dissipation and is strictly decreasing
in time. All non-monotonic behavior at fixed $\lambda$ is therefore encoded in
the amplitude function $A(\lambda,t)$, which arises solely from scale-free
transport.

\paragraph{Intuitive picture.}
Intuitively, scale-free transport continuously transfers error mass from
higher-$\lambda$ modes, which are learned and dissipated rapidly, toward
lower-$\lambda$ modes.
As a result, even as high-resolution errors decay, part of this error is
temporarily deposited onto lower-resolution components, raising their spectral
error density.

At the same time, dissipation at scale $\lambda$ is proportional to
$\lambda\,\varepsilon(\lambda,t)$, and is therefore weaker at small $\lambda$.
The mismatch between fast transport from high $\lambda$ and delayed dissipation
at low $\lambda$ naturally produces a transient amplification before eventual
decay.

Differentiating the amplitude function at fixed $\lambda$ yields
\begin{equation}
\partial_t A(\lambda,t)
=
-\,c_0\,\lambda\,t^{-c_0-1}\,
\varepsilon_0'(\lambda t^{-c_0}).
\label{eq:sec7_amp_derivative}
\end{equation}
Hence, whenever the initial spectrum is decreasing,
$\varepsilon_0'(u)<0$, and $c_0>0$, the amplitude grows in time:
\[
\partial_t A(\lambda,t)>0.
\]
Consequently, the spectral error density $\varepsilon(\lambda,t)$ can increase
transiently at fixed $\lambda$, despite overall dissipation and monotone decay
of the training loss.

\subsection{Double descent from transport-induced amplification in the mismatch window}
\label{sec:double_descent_mismatch}

We now formalize how transport-induced amplitude growth leads to
double descent when combined with finite-sample spectral mismatch.
The key point is that, although the training loss is strictly monotone,
the test loss probes a different quadratic form that selectively weights
high-resolution spectral components.

\paragraph{Spectral localization of finite-sample mismatch.}
Recall that the test loss can be written as
\[
L_{\mathrm{te}}(t)
=
L_{\mathrm{tr}}(t)
+
\frac12 \langle e_t, \Delta e_t \rangle,
\qquad
\Delta := P_{\mathrm{te}} - P_{\mathrm{tr}},
\]
where $\Delta$ is the mismatch operator.
Classical results on empirical covariance and integral operators imply that,
for continuous data distributions and infinite-rank population operators,
$\Delta$ is generically supported on high-resolution modes,
corresponding to small eigenvalues $\lambda$
of the population operator.

At the coarse-grained level, this implies that there exists a resolution
interval
\[
\lambda \in [\lambda_-, \lambda_+],
\]
which we refer to as the \emph{mismatch window},
such that the dominant contribution to
$\langle e_t, \Delta e_t \rangle$
comes from spectral components in this range.
Outside this window, either the modes are well-aligned between train and test
or their contribution is suppressed by dissipation.

\paragraph{Mismatch-weighted test loss and non-monotonicity.}
The mismatch contribution to the test loss can be approximated,
to leading order, by restricting to the mismatch window:
\[
\langle e_t, \Delta e_t \rangle
\;\approx\;
\int_{\lambda_-}^{\lambda_+}
w(\lambda)\,
\varepsilon(\lambda,t)\, d\lambda,
\]
where $w(\lambda)\ge 0$ denotes the coarse-grained spectral weight induced by the
finite-sample mismatch operator $\Delta=P_{\mathrm{te}}-P_{\mathrm{tr}}$ in the
instantaneous spectral basis of $M(t)$, so that $w(\lambda)\,d\lambda$ represents
the leading-order contribution of modes with eigenvalues in $[\lambda,\lambda+d\lambda]$
to the quadratic form $\langle e_t,\Delta e_t\rangle$. Formally, $w(\lambda)$ is defined by coarse-graining the diagonal spectral density
of the mismatch operator $\Delta$ in the instantaneous eigenbasis of $M(t)$,
i.e.
\[
w(\lambda)
\;:=\;
\sum_{u:\,\lambda_u(t)\in[\lambda,\lambda+d\lambda)}
\langle \phi_u(t),\Delta\,\phi_u(t)\rangle \,/\, d\lambda,
\]
so that $\langle e_t,\Delta e_t\rangle \approx \int w(\lambda)\,\varepsilon(\lambda,t)\,d\lambda$
to leading order.

Substituting the factorized form of $\varepsilon$ gives
\[
\int_{\lambda_-}^{\lambda_+}
w(\lambda)\,
A(\lambda,t)\,
\exp\!\left(
-\frac{2\lambda}{1+c_0}\,t
\right)
\, d\lambda.
\]

If the mismatch window $[\lambda_-,\lambda_+]$ intersects the region of
spectral coordinates for which $\partial_t A(\lambda,t)>0$,
then the integrand exhibits a competition between
transport-induced amplification and dissipation.
At early times, the growth of $A(\lambda,t)$ dominates,
leading to an increase of the mismatch contribution.
At later times, exponential dissipation takes over,
forcing eventual decay.

\paragraph{Sufficient condition for double descent.}
We therefore obtain the following sufficient condition:
if there exists a time interval and a non-negligible spectral region
\[
\lambda \in [\lambda_-,\lambda_+]
\quad \text{such that} \quad
\partial_t A(\lambda,t)>0,
\]
then the test loss $L_{\mathrm{te}}(t)$ is necessarily non-monotone,
even though the training loss $L_{\mathrm{tr}}(t)$ decreases monotonically.
This non-monotonicity manifests as the classical double-descent curve.

Importantly, this mechanism is entirely in-domain.
It does not rely on distribution shift or adversarial effects,
but follows generically from the interaction between
renormalizable spectral transport and finite-sample spectral mismatch.



\section{Empirical Spectral-Shell Analysis of LLM Training}
\label{sec:experiments}

This section empirically examines whether the coarse-grained spectral-shell
structures assumed by the GRSD framework are observable in realistic large
language model (LLM) training checkpoints. Our goal is not to infer microscopic
transport mechanisms or fit scaling exponents, but to validate the existence of
the shell-level structures required by the theory.

\subsection{Experimental setup and error definition}
\label{sec:exp_setup}

\paragraph{Models and checkpoints.}
We analyze publicly released checkpoints from the Pythia model family, including
models with 70M, 410M, and 7B parameters. Following the standard Pythia release
protocol, checkpoints are provided at regular intervals of approximately
1k optimization steps. All checkpoints analyzed in this section are taken
directly from the official releases; no selection based on loss values or
convergence criteria is performed.

\paragraph{Error observable.}
Let $q_\theta(y\mid x)$ denote the model-predicted conditional distribution.
Throughout this section, the error signal is defined as the negative
log-probability of the ground-truth token,
\begin{equation}
e(x) \;:=\; -\log q_\theta(y^*\mid x),
\end{equation}
and all spectral quantities are computed from the squared error $e(x)^2$.

\paragraph{Justification via quadratic approximation of KL.}
For language modeling with one-hot (or near one-hot) supervision, the ground-truth
conditional distribution satisfies $p(y^*\mid x)\approx 1$. In this setting, the
KL divergence reduces to the negative log-likelihood,
\begin{equation}
\mathrm{KL}(p\|q_\theta)
= -\log q_\theta(y^*\mid x) + \mathrm{const}.
\end{equation}
Moreover, when $q_\theta(y^*\mid x)\approx 1$, a second-order expansion yields the
local quadratic approximation
\begin{equation}
\mathrm{KL}(p\|q_\theta)
\;\approx\;
\tfrac12\,\big(-\log q_\theta(y^*\mid x)\big)^2.
\end{equation}
Accordingly, the squared log-probability provides a locally equivalent quadratic
error measure, consistent with the mean-squared-error framework underlying the
spectral-shell analysis. A more detailed derivation is provided in
Appendix~\ref{app:kl_mse}.

\subsection{Spectral estimation via stochastic Lanczos}
\label{sec:lanczos}

\paragraph{Motivation.}
The operator $M = J_\theta J_\theta^*$ governing error dynamics acts on a
high-dimensional function space, making explicit eigendecomposition infeasible
for modern LLMs. To estimate shell-resolved spectral quantities at scale, we
employ a stochastic Lanczos method, which probes the spectrum of $M$ using
operator--vector products without explicitly forming $M$.

\paragraph{Shell-energy estimation.}
The Lanczos procedure produces a low-dimensional tridiagonal approximation whose
spectral measure approximates the projection of the error onto the eigenmodes of
$M$. The resulting spectral weights are aggregated into logarithmic bins in the
spectral coordinate $\lambda$, yielding the log-shell--integrated quadratic error
energies
\begin{equation}
E_\alpha
\;=\;
\sum_{u:\,\lambda_u\in[\lambda_\alpha,\lambda_{\alpha+1})} g_u^2,
\end{equation}
which serve as the primary observables throughout this section. All reported
curves correspond to $\log E_\alpha$ plotted against the shell centers
$\log\lambda_\alpha$.

\paragraph{Data-induced averaging.}
To obtain stable shell-level estimates, all spectral quantities are averaged over
multiple independent data batches. Specifically, each estimate aggregates results
from approximately 50 batches, each consisting of 128 context windows with batch
size 4, corresponding to roughly $2.5\times 10^5$ tokens in total. No additional
stochasticity is introduced in the Lanczos procedure itself; all averaging arises
from the data distribution.

Implementation details and pseudocode for the Lanczos-based shell-energy
estimation are provided in Appendix~\ref{app:lanczos}.

\subsection{Shell-energy structure at a fixed checkpoint}
\label{sec:fixed_ckpt}

\begin{figure}[t]
    \centering
    \begin{subfigure}[t]{0.32\textwidth}
        \centering
        \includegraphics[width=\textwidth]{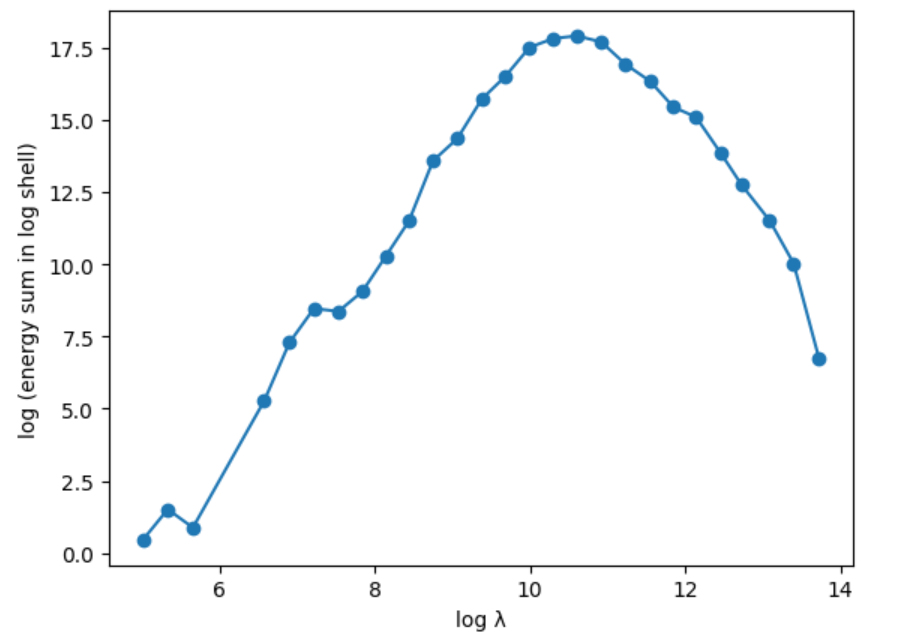}
        \caption{Shell-energy profile at a fixed checkpoint.}
        \label{fig:7b_shell_profile}
    \end{subfigure}
    \hfill
    \begin{subfigure}[t]{0.32\textwidth}
        \centering
        \includegraphics[width=\textwidth]{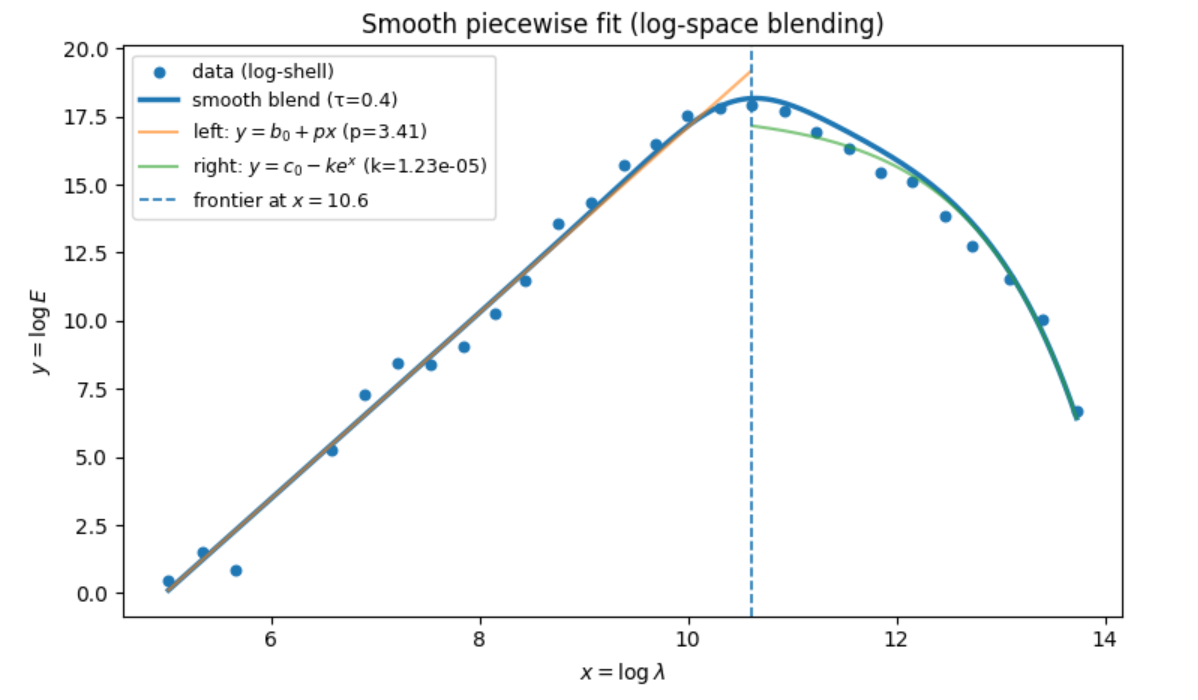}
        \caption{Piecewise fit with smooth blending.}
        \label{fig:7b_piecewise_fit}
    \end{subfigure}
    \hfill
    \begin{subfigure}[t]{0.32\textwidth}
        \centering
        \includegraphics[width=\textwidth]{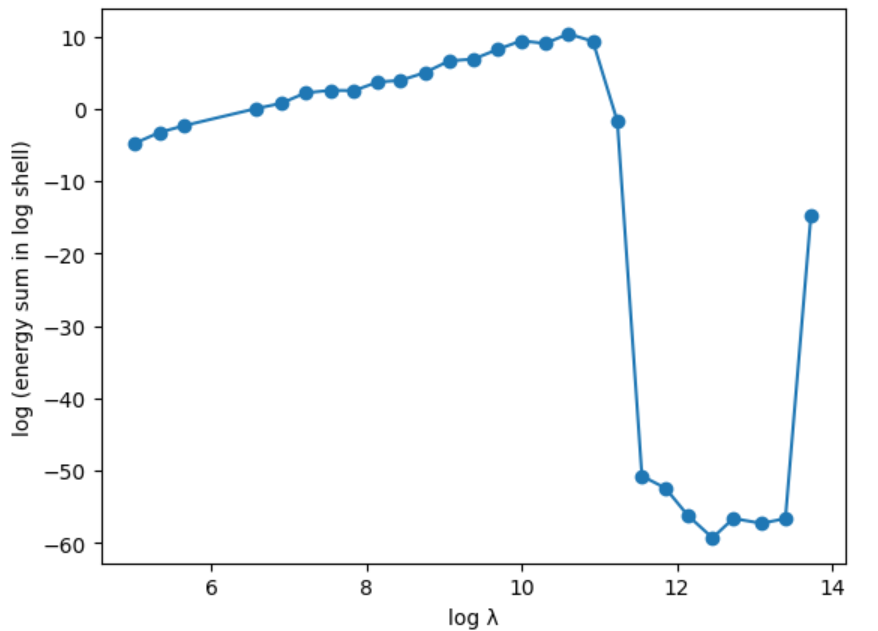}
        \caption{Minimum shell energy.}
        \label{fig:7b_min_energy}
    \end{subfigure}
    \caption{
    Spectral-shell structure of log-probability error for the 7B Pythia model at a representative checkpoint.
    \textbf{(a)} Minimum energy within each log shell exhibits an abrupt collapse beyond a finite resolution,
    indicating a sharp dissipation-dominated cutoff rather than an extended power-law tail.
    \textbf{(b)} A piecewise model---log-linear (power-law) behavior on the low-resolution side and exponential
    suppression beyond the spectral frontier---provides a good qualitative description of the observed profile.
    The two regimes are combined via a smooth log-space blending for visualization.
    \textbf{(c)} The full log-shell energy distribution forms a robust unimodal profile, with a well-defined
    spectral frontier separating transport-dominated and dissipation-dominated regimes.
    }
    \label{fig:7b_shell_structure}
\end{figure}

\paragraph{Theoretical expectation.}
As derived in Eq~\ref{eq:shell_energy_powerlaw_example}, the shell-integrated quadratic
error energy admits a factorized form beyond a finite resolution frontier,
\begin{equation}
E_\alpha(t)
\;\approx\;
A(\lambda_\alpha,t)\,
\exp\!\big(-\kappa(t)\,\lambda_\alpha\big),
\qquad
\lambda_\alpha \ \text{beyond the frontier},
\end{equation}
where $A(\lambda,t)$ is a coarse-grained amplitude inherited from transport and
initial conditions. In particular, when the initial error spectrum is
approximately scale-free over a range of resolutions, the amplitude
$A(\lambda,t)$ may exhibit an effective power-law dependence on $\lambda$ over
that range.

These observations imply the existence of a well-defined \emph{resolution
frontier} separating two qualitatively distinct regimes. On the low-$\lambda$
side of the frontier, the shell-integrated error energy varies smoothly across
logarithmic shells and is well approximated by an affine trend in log--log
coordinates, indicating a slowly varying, transport-dominated structure. As
$\lambda$ increases, the energy profile reaches a single pronounced maximum,
beyond which the behavior changes abruptly.

At fixed training checkpoints, the log-shell energy $E_\alpha$ consistently
exhibits a unimodal distribution across resolutions: energy accumulates gradually
from low $\lambda$, peaks at an intermediate scale, and then decays rapidly at
higher resolutions. To probe the high-resolution regime more directly, we
examine the minimum error energy within each log shell. Beyond a finite
resolution threshold, the minimum shell energy collapses sharply to a numerical
floor, producing a clear cutoff on the right side of the spectrum. This abrupt
suppression is incompatible with any extended power-law tail and instead
indicates a dissipation-dominated regime past the spectral frontier.

\subsection{Comparison across checkpoints: spectral-front displacement}
\label{sec:ckpt_compare}

\begin{figure}[t]
    \centering

    \begin{subfigure}[t]{0.48\textwidth}
        \centering
        \includegraphics[width=\textwidth]{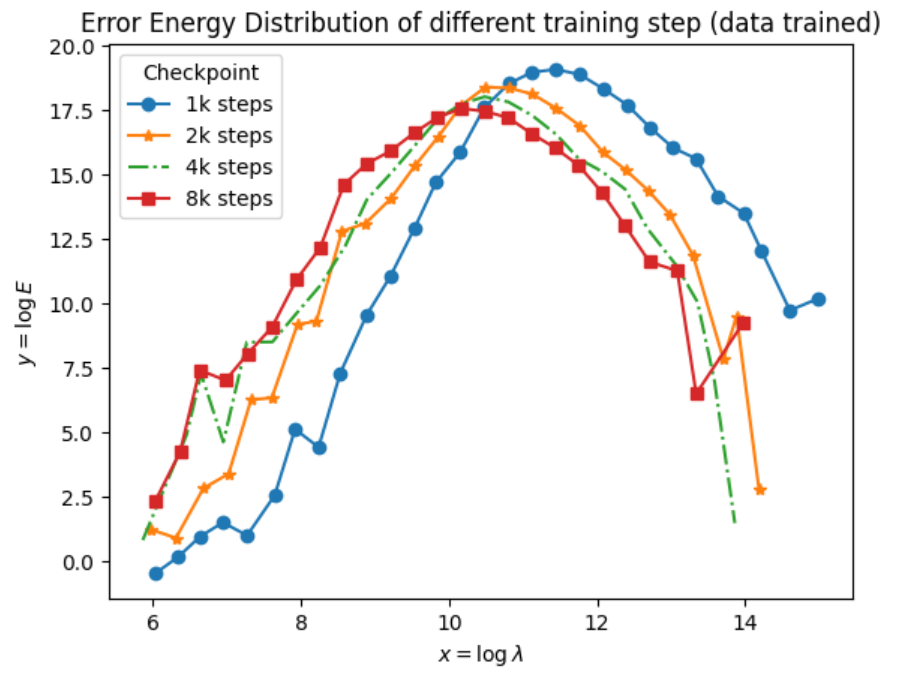}
        \caption{Raw log-shell energy profiles.}
        \label{fig:ckpt_compare_raw}
    \end{subfigure}
    \hfill
    \begin{subfigure}[t]{0.48\textwidth}
        \centering
        \includegraphics[width=\textwidth]{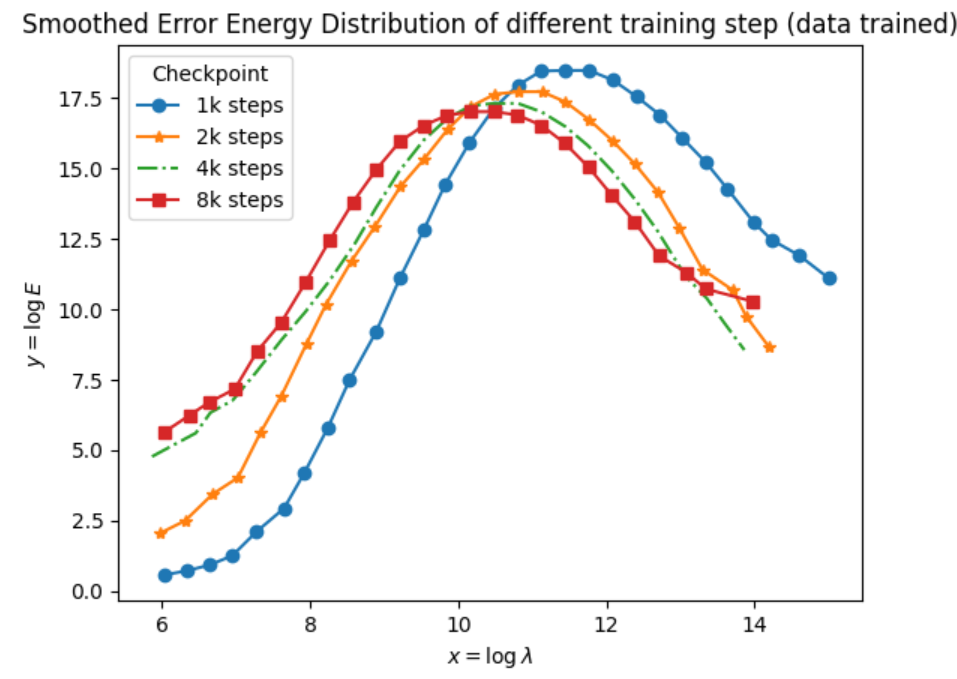}
        \caption{Smoothed log-shell energy profiles.}
        \label{fig:ckpt_compare_smooth}
    \end{subfigure}

    \caption{
    Comparison of log-shell error energy distributions across training checkpoints
    (1k, 2k, 4k, and 8k steps) for a fixed model.
    \textbf{(a)} Raw shell-energy profiles exhibit a consistent unimodal shape,
    while the location of the spectral peak and the extent of the high-resolution
    tail vary systematically across checkpoints.
    \textbf{(b)} Smoothed profiles suppress shell-level fluctuations and more
    clearly reveal a leftward displacement of the spectral peak and a contraction
    of the high-resolution tail as training progresses. Across checkpoints, the
    low-$\lambda$ side of the spectrum remains approximately parallel in log--log
    coordinates, indicating coherent reweighting of shell energies rather than
    independent decay at fixed resolutions.
    }
    \label{fig:ckpt_compare}
\end{figure}

We next compare log-shell energy profiles across multiple training checkpoints
for a fixed model. As shown in Fig.~\ref{fig:ckpt_compare}, the shell-energy
distributions at different checkpoints all retain a pronounced unimodal shape,
indicating a stable coarse-grained spectral organization throughout training.
However, the location of the peak and the extent of the high-resolution tail
shift systematically across checkpoints.

In particular, as training progresses, the position of the spectral peak moves
toward smaller $\lambda$, and the high-resolution side of the distribution
contracts accordingly. This behavior is more clearly revealed in the smoothed
profiles, which suppress shell-level noise while preserving the global shape.
The observed shift reflects a gradual displacement of the effective resolution
frontier rather than the emergence of qualitatively new spectral features.

At the same time, the low-$\lambda$ side of the spectrum exhibits coherent
reweighting across checkpoints. Over a broad range of low resolutions, the
shell-energy curves remain approximately parallel in log--log coordinates,
while their overall amplitudes increase with training. This indicates that error
energy is redistributed across spectral shells in a correlated manner, rather
than decaying independently at fixed resolutions.

\subsection{Cross-model consistency across scales}
\label{sec:cross_model}
\begin{figure}[t]
    \centering

    \begin{subfigure}[t]{0.48\textwidth}
        \centering
        \includegraphics[width=\textwidth]{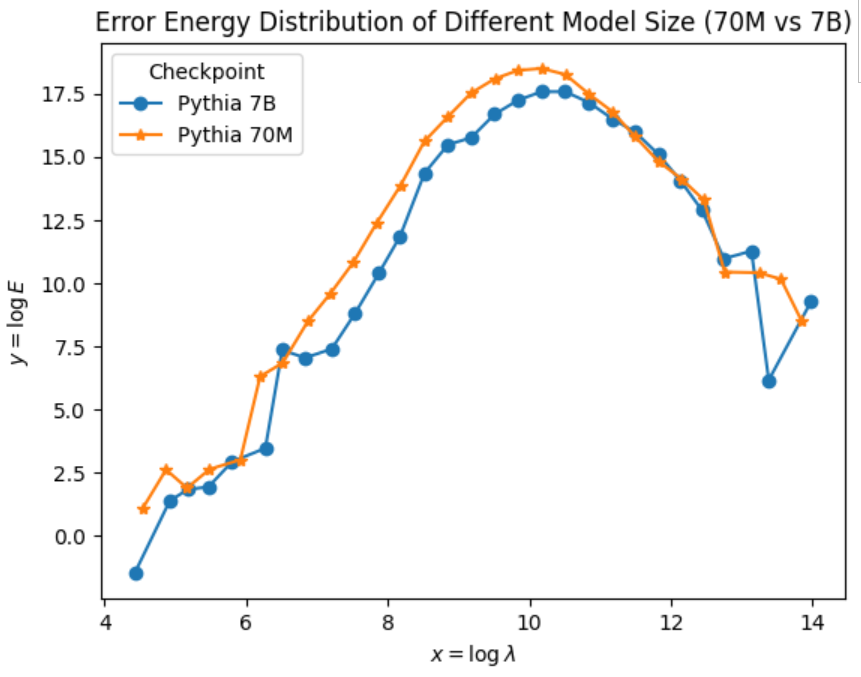}
        \caption{Raw log-shell energy profiles.}
        \label{fig:cross_model_raw}
    \end{subfigure}
    \hfill
    \begin{subfigure}[t]{0.48\textwidth}
        \centering
        \includegraphics[width=\textwidth]{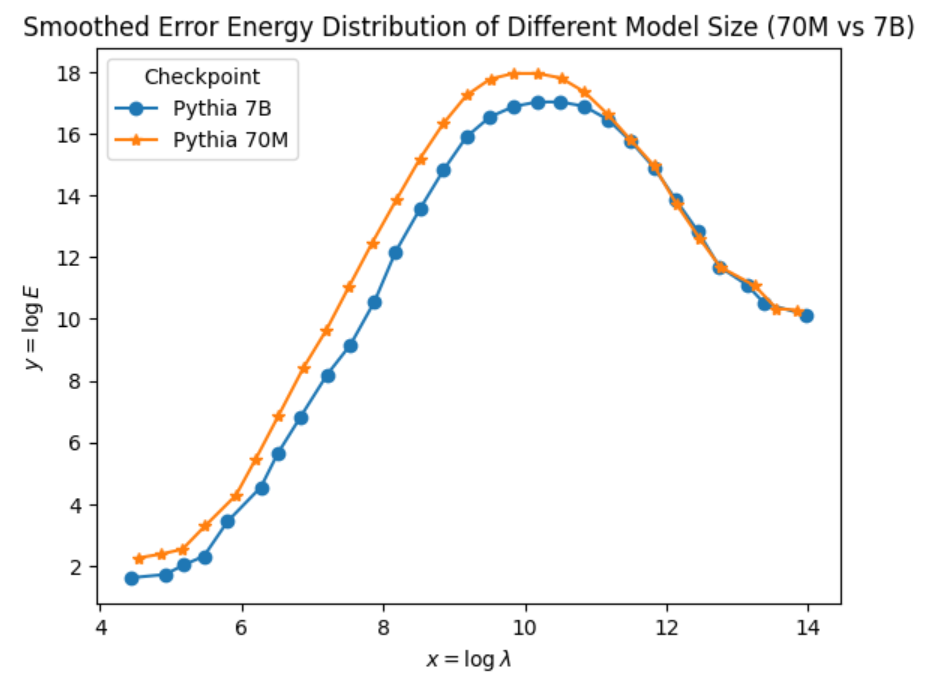}
        \caption{Smoothed log-shell energy profiles.}
        \label{fig:cross_model_smooth}
    \end{subfigure}

    \caption{
    Cross-model comparison of log-shell error energy distributions for Pythia
    models with 70M and 7B parameters at the same checkpoint index.
    \textbf{(a)} Raw shell-energy profiles show that both models exhibit a similar
    unimodal structure, with energy rising from low $\lambda$, peaking at an
    intermediate resolution, and decaying at higher resolutions.
    \textbf{(b)} Smoothed profiles suppress shell-level fluctuations and highlight
    the close alignment of the two distributions across a broad range of
    resolutions. Differences between model scales appear primarily as overall
    amplitude shifts and modest displacements of the spectral peak, while the
    low-$\lambda$ side of the spectrum remains approximately parallel in log--log
    coordinates.
    }
    \label{fig:cross_model}
\end{figure}

We finally compare shell-energy profiles across models of different scales.
Fig.~\ref{fig:cross_model} shows the log-shell error energy distributions for
Pythia models with 70M and 7B parameters at the same checkpoint index. Despite
the substantial difference in model size, the two profiles exhibit strikingly
similar coarse-grained structure.

In both cases, the shell-energy distribution forms a pronounced unimodal shape,
with energy rising smoothly from low $\lambda$, peaking at an intermediate
resolution, and decaying at higher resolutions. The smoothed profiles further
highlight this similarity: across a broad range of resolutions, the two curves
closely track each other, differing primarily by an overall amplitude shift and
a modest displacement of the peak location.

Notably, the low-$\lambda$ side of the spectrum remains approximately parallel in
log--log coordinates across model scales, indicating that the relative
distribution of error energy across coarse spectral shells is largely preserved.
Differences between the 70M and 7B models manifest mainly in the overall scale of
energy and the precise position of the spectral peak, rather than in the
emergence of qualitatively new spectral features.

This cross-model consistency suggests that the observed shell-level organization
reflects a scale-robust property of training dynamics, rather than an artifact of
a particular model size.

\subsection{Summary of empirical observations}
\label{sec:exp_summary}

In summary, our experiments establish the following empirical facts:
(i) squared log-probability error admits a unimodal shell-energy profile at fixed
checkpoints; (ii) a sharp high-resolution frontier exists beyond which shell
energy is rapidly suppressed; (iii) the frontier position and low-resolution
amplitudes vary coherently across checkpoints; and (iv) these shell-level
structures persist across model scales from 70M to 7B parameters.

These observations support the structural assumptions underlying renormalizable
spectral-shell dynamics without relying on any specific microscopic transport
model.

\section{Related Work}

\paragraph{Neural scaling laws.}
Empirical studies have established that neural networks display remarkably
regular power-law relationships between compute, model size, dataset size,
and achievable loss \citep{kaplan2020scaling, henighan2020scaling, 
hestness2017deep, hoffmann2022training, hernandez2021scaling}.
Recent work has investigated both the emergence of such scaling behavior and
its theoretical underpinnings, including dynamical models of loss evolution
\citep{bordelon2024dynamical} and the role of data pruning 
\citep{sorscher2022beyond}.
Precision scaling laws have also been explored in the context of
architectural and quantization constraints \citep{kumar2024scaling}.

\paragraph{Feature learning and spectral bias.}
A major line of work studies how neural networks acquire hierarchical
representations and exhibit spectral preference 
\citep{rahaman2019spectral, bordelon2023feature}.
Kernel-based analyses of learning dynamics 
\citep{bietti2021inductive, canatar2022kernel} 
and infinite-width approximations 
\citep{jacot2018neural, lee2019wide, yang2021tensor} 
have contributed significantly to understanding the transition between lazy
training and representation learning.  
Recent studies have also revealed consistency of learned features across widths
\citep{vyas2023feature} and the spectral evolution of networks
\citep{wang2023spectralevolution}.

\paragraph{Optimization dynamics and stability.}
The geometry of optimization landscapes, the effect of batch size, and
training stability have been examined extensively
\citep{keskar2017large, ghorbani2019investigation}.
Work on pruning and compression
\citep{rosenfeld2021predictability, han2015deep, blalock2020state,
lecun1990optimal, molchanov2017variational, lee2019snip, wang2020picking,
han2016deep, nagel2021up, frantar2022gptq, dettmers2022llm}
has illuminated how spectral structure interacts with parameter sparsity and
low-precision computation.
Theoretical connections between spectral evolution, implicit bias, 
and neural dynamics continue to be an active area of research
\citep{domine2024lazy, zhang2025generalized}.

\paragraph{Perturbation theory and adiabatic analysis.}
Our spectral formulation draws on classical operator perturbation theory
\citep{kato2012short} and its analogues in quantum adiabatic evolution
\citep{Zwiebach_Adiabatic}.
These tools formalize how eigenfunctions and eigenvalues evolve under 
smooth or weakly coupled updates, providing a principled foundation for the 
drift–dissipation dynamics developed in this work.

\section{Conclusion}

This work proposed a spectral--shell framework for understanding neural
scaling laws, feature learning, and double-descent phenomena directly from
the operator-level dynamics induced by gradient descent.
Starting from the exact function-space evolution
$\dot e_t=-M(t)e_t$, we derived a rigorous modewise formulation and showed
that, upon logarithmic coarse-graining, shell-internal interactions cancel
identically at the level of quadratic error energy.
As a result, the macroscopic evolution of error is governed entirely by
dissipation and inter-shell energy exchange.

The central modeling ingredient of this paper is a \emph{renormalizable
spectral-shell dynamics} assumption: after coarse-graining, the cumulative
effect of microscopic interactions can be summarized by a controlled net
energy flux across shell boundaries.
Under an effective power-law form of this renormalized shell-level flux,
the shell dynamics admit a self-similar high-resolution tail with a moving
resolution frontier.
This structure yields explicit scaling-law behavior for the total loss and
provides a unified explanation of neural scaling laws and double descent.

A key conceptual outcome of this framework is a unified view of lazy
(NTK-like) training and feature learning.
When inter-shell transport is negligible, shell energies decay independently
under dissipation, recovering classical kernel dynamics.
When transport is active, spectral energy is redistributed across resolutions
before being dissipated, inducing representation learning.
Both regimes—and the continuum between them—are governed by the same
shell-energy bookkeeping, differing only in the effective shell-flux
interface.

Importantly, the analysis does not assume the existence of a continuous
spectrum or a vanishing shell spacing.
All continuum partial differential equations appearing in the paper serve
only as convenient local approximations of discrete shell-energy dynamics
over resolution ranges where shells are sufficiently dense.
All scaling-law conclusions can be equivalently interpreted at the discrete
shell level.

\paragraph{Future directions.}
Several important questions remain open:

\begin{enumerate}
\item \textbf{Multi-task and multi-distribution learning.}
  The present analysis focuses on a single task and a single data
  distribution.
  Extending spectral-shell dynamics to multi-task settings—where different
  task operators may not share eigenbases and may interact through shared
  representation drift—could reveal new mechanisms of transfer, interference,
  and modular generalization.

\item \textbf{Origin of renormalizability and power-law shell transport.}
  In this work, the renormalizable shell-flux interface and its effective
  power-law form are treated as macroscopic assumptions.
  A natural next step is to understand their qualitative and quantitative
  origins from operator-level structure.
  Possible mechanisms include locality of spectral coupling, effective
  one-directional energy transfer across resolution shells, and the
  suppression of long-range interactions by stochastic optimization noise.
  Clarifying when and why such mechanisms produce renormalizable and
  power-law shell dynamics would substantially deepen the theoretical
  foundations of the framework.

\item \textbf{Modeling optimizers and learning-rate schedules.}
  The present formulation is developed for gradient-based optimization with
  an effective time reparameterization.
  Momentum, adaptive methods, and second-order or approximate natural-gradient
  schemes can modify both dissipation and transport.
  Developing spectral-shell descriptions of these optimizers—and of common
  learning-rate schedules—may explain optimizer-dependent variations in
  scaling behavior.

\item \textbf{Connecting $J_\theta$ to network architecture.}
  The Jacobian operator $J_\theta$ mediates how parameters generate spectral
  structure in function space.
  Understanding how architectural features such as depth, width, and
  parameterization shape the induced shell dynamics could connect model
  scaling, representational capacity, and achievable loss within a unified
  operator-theoretic framework.
\end{enumerate}

Overall, the spectral-shell perspective developed here isolates a minimal and
robust macroscopic structure underlying neural training dynamics.
By separating exact shell-level conservation laws from coarse-grained flux
assumptions, it provides a flexible foundation for understanding scaling
phenomena while leaving room for future work on their microscopic origins.

\bibliographystyle{plainnat}
\bibliography{references}

\appendix

\section{KL-to-MSE justification for log-probability error}
\label{app:kl_mse}

This appendix justifies the quadratic error observable used in the empirical
spectral-shell analysis. In language modeling with one-hot supervision, the
ground-truth conditional distribution is concentrated at the correct token
$y^*$, while the model predicts $q_\theta(\cdot\mid x)$.
For a fixed context $x$, the token-level KL divergence satisfies
\[
\mathrm{KL}\!\big(p(\cdot\mid x)\,\|\,q_\theta(\cdot\mid x)\big)
= \sum_y p(y\mid x)\log\frac{p(y\mid x)}{q_\theta(y\mid x)}
= -\log q_\theta(y^*\mid x) + C(x),
\]
where $C(x)$ depends only on $p(\cdot\mid x)$ and is independent of $\theta$.

Define the (log-probability) error observable
\[
e(x) \;:=\; -\log q_\theta(y^*\mid x).
\]
When the model is already assigning high probability to the correct token,
$q_\theta(y^*\mid x)\approx 1$, we have $e(x)\approx 0$ and a local quadratic
approximation is appropriate. Consider $q=\exp(-e)$, so that $q\approx 1$ iff
$e\approx 0$. Expanding $\exp(-e)$ around $e=0$ gives
\[
q = e^{-e} = 1 - e + \tfrac12 e^2 + O(e^3),
\qquad\Rightarrow\qquad
1-q = e + O(e^2).
\]
Thus, in a neighborhood of $q=1$, any smooth loss functional that is minimized at
$q=1$ admits a second-order expansion in $e$. In particular, the KL divergence
above equals $e$ up to an additive constant, and its \emph{local quadratic surrogate}
takes the form
\[
\mathrm{KL}(p\|q_\theta) \;\approx\; \tfrac12\, e(x)^2
\qquad\text{for } e(x)\ll 1 \; (q_\theta(y^*\mid x)\approx 1).
\]
Therefore, using the squared log-probability $e(x)^2$ as a mean-squared error
observable is consistent with the quadratic-loss framework that underlies the
spectral-shell bookkeeping: the induced operator dynamics and shell-energy
definitions apply directly once the error signal is taken to be $e(x)$ and the
energy density is built from $e(x)^2$.

\paragraph{Remark (what is measured empirically).}
All empirical spectra in this paper are computed from the squared error
$e(x)^2 = \big(-\log q_\theta(y^*\mid x)\big)^2$ and then aggregated into
logarithmic spectral shells.


\section{Spectral estimation via (stochastic) Lanczos and shell binning}
\label{app:lanczos}

This appendix describes the Lanczos-based estimator used to obtain log-shell
energies without forming the operator $M=J_\theta J_\theta^*$ explicitly.

\subsection{Goal: shell energies from the spectral measure of $M$}
Fix a checkpoint $\theta$ and let $M := J_\theta J_\theta^*$ be the PSD operator
governing error dynamics in function space. Let $e$ denote the error signal
(Section~\ref{sec:setup}; in our experiments $e(x)=-\log q_\theta(y^*\mid x)$).
Formally diagonalizing $M\phi_u=\lambda_u\phi_u$ and expanding $e=\sum_u g_u\phi_u$,
the shell-integrated energy is
\[
E_\alpha \;:=\; \sum_{u:\,\lambda_u\in[\lambda_\alpha,\lambda_{\alpha+1})} g_u^2.
\]
Equivalently, define the spectral measure $\mu_e$ of $M$ induced by $e$ so that
for any test function $f$,
\[
\langle e, f(M)e\rangle = \int f(\lambda)\, d\mu_e(\lambda),
\qquad
d\mu_e(\lambda) = \sum_u g_u^2\,\delta(\lambda-\lambda_u)\,d\lambda.
\]
Then $E_\alpha$ is exactly the mass of $\mu_e$ in the bin
$[\lambda_\alpha,\lambda_{\alpha+1})$.

\subsection{Lanczos tridiagonalization and quadrature}
Lanczos constructs, from repeated applications of $M$ to vectors, an orthonormal
basis of the Krylov subspace $\mathcal{K}_m(M,v_0)=\mathrm{span}\{v_0,Mv_0,\dots,M^{m-1}v_0\}$
and a symmetric tridiagonal matrix $T_m\in\mathbb{R}^{m\times m}$ such that
\[
M V_m \;\approx\; V_m T_m,
\qquad
V_m=[v_0,\dots,v_{m-1}],\quad V_m^*V_m=I.
\]
We choose the starting vector to be the normalized error direction
\[
v_0 := \frac{e}{\|e\|}.
\]
Let $T_m = U \,\mathrm{diag}(\tilde\lambda_i)\, U^*$ be the eigendecomposition of
the tridiagonal matrix. Then Gauss quadrature yields an approximation of the
spectral measure $\mu_e$ by the discrete measure
\[
\mu_e(d\lambda)\;\approx\;\|e\|^2 \sum_{i=1}^m w_i\,\delta(\lambda-\tilde\lambda_i)\,d\lambda,
\qquad
w_i := (U_{1i})^2,
\]
where $U_{1i}$ is the first component of the $i$-th eigenvector of $T_m$.

\subsection{Log-shell binning}
Given logarithmic bin edges $\{\lambda_\alpha\}_{\alpha=0}^{B}$, define the estimated
shell energies by accumulating quadrature masses:
\[
\widehat E_\alpha
\;:=\;
\|e\|^2 \sum_{i:\,\tilde\lambda_i\in[\lambda_\alpha,\lambda_{\alpha+1})} w_i.
\]
We report $\log \widehat E_\alpha$ as a function of $\log\lambda_\alpha^{\mathrm{ctr}}$,
where $\lambda_\alpha^{\mathrm{ctr}}$ is the geometric bin center.

\subsection{Data-induced averaging (no extra randomness)}
We do not rely on a global random seed. Variance reduction is achieved by
averaging the resulting shell estimates across many independent data batches:
each batch provides an empirical error vector $e$ (computed from a large set of
tokens), and the above Lanczos+binning pipeline produces $\widehat E_\alpha$ for
that batch; we then average $\widehat E_\alpha$ over batches.

\subsection{Pseudo-code}
\begin{algorithm}[t]
\caption{Lanczos shell-energy estimation at a fixed checkpoint}
\label{alg:lanczos_shell_energy}
\begin{algorithmic}[1]
\Require Operator-vector product routine $u\mapsto Mu$; error vector $e$;
Lanczos steps $m$; log-shell edges $\{\lambda_\alpha\}_{\alpha=0}^{B}$.
\Ensure Estimated shell energies $\widehat E_\alpha$.
\State $v_0 \leftarrow e/\|e\|$; $\beta_0\leftarrow 0$; $v_{-1}\leftarrow 0$
\For{$k=0,\dots,m-1$}
    \State $w \leftarrow Mv_k - \beta_k v_{k-1}$
    \State $\alpha_k \leftarrow \langle v_k,w\rangle$
    \State $w \leftarrow w - \alpha_k v_k$
    \State $\beta_{k+1} \leftarrow \|w\|$
    \If{$\beta_{k+1}=0$} \textbf{break} \EndIf
    \State $v_{k+1} \leftarrow w/\beta_{k+1}$
\EndFor
\State Form tridiagonal $T$ with diagonal entries $\alpha_k$ and off-diagonals $\beta_{k+1}$
\State Compute eigendecomposition $T = U\,\mathrm{diag}(\tilde\lambda_i)\,U^*$
\State Set weights $w_i \leftarrow (U_{1i})^2$
\State Initialize $\widehat E_\alpha \leftarrow 0$ for all $\alpha=0,\dots,B-1$
\For{$i=1,\dots,m$}
    \State Find $\alpha$ such that $\tilde\lambda_i\in[\lambda_\alpha,\lambda_{\alpha+1})$
    \If{such $\alpha$ exists}
        \State $\widehat E_\alpha \leftarrow \widehat E_\alpha + \|e\|^2\, w_i$
    \EndIf
\EndFor
\State \Return $\{\widehat E_\alpha\}_{\alpha=0}^{B-1}$
\end{algorithmic}
\end{algorithm}

\paragraph{Implementation note.}
The only required primitive is the product $u\mapsto Mu = J_\theta(J_\theta^* u)$.
In modern autodiff frameworks this can be implemented via a VJP followed by a JVP,
without constructing $J_\theta$ or $M$ explicitly.

\end{document}